\documentclass{article}
\usepackage{ijcai18}

\usepackage{times}
\usepackage{xcolor}
\usepackage{soul}
\usepackage[utf8]{inputenc}
\usepackage[small]{caption}

\usepackage{helvet}
\usepackage{courier}
\usepackage{url}
\usepackage{graphicx}
\usepackage{amsthm}
\usepackage{amsmath}
\usepackage[linesnumbered,ruled]{algorithm2e}
\usepackage{multirow}
\usepackage{txfonts}
\usepackage{subcaption}
\usepackage{comment}
\usepackage{makecell}
\usepackage{bbm}

\newtheorem{definition}{Definition}
\newtheorem{theorem}{Theorem}
\newtheorem{proposition}{Proposition}

\newtheorem{corollary}{Corollary}
\newtheorem{problem}{Problem}

\makeatletter
\newcommand{\subalign}[1]{%
  \vcenter{%
    \Let@ \restore@math@cr \default@tag
    \baselineskip\fontdimen10 \scriptfont\tw@
    \advance\baselineskip\fontdimen12 \scriptfont\tw@
    \lineskip\thr@@\fontdimen8 \scriptfont\thr@@
    \lineskiplimit\lineskip
    \ialign{\hfil$\m@th\scriptstyle##$&$\m@th\scriptstyle{}##$\crcr
      #1\crcr
    }%
  }
}
\makeatother

\usepackage{siunitx}
\title{Achieving Non-Discrimination in Prediction}

\author{Lu Zhang, Yongkai Wu, and Xintao Wu\\
University of Arkansas\\
\{lz006,yw009,xintaowu\}@uark.edu}

\begin{document}

\maketitle

\begin{abstract}
Discrimination-aware classification is receiving an increasing attention in data science fields. The pre-process methods for constructing a discrimination-free classifier first remove discrimination from the training data, and then learn the classifier from the cleaned data. However, they lack a theoretical guarantee for the potential discrimination when the classifier is deployed for prediction. In this paper, we fill this gap by mathematically bounding the probability of the discrimination in prediction being within a given interval in terms of the training data and classifier. We adopt the causal model for modeling the data generation mechanism, and formally defining discrimination in population, in a dataset, and in prediction. We obtain two important theoretical results: (1) the discrimination in prediction can still exist even if the discrimination in the training data is completely removed; and (2) not all pre-process methods can ensure non-discrimination in prediction even though they can achieve non-discrimination in the modified training data. %The experiment results show the effectiveness of our two-phase framework.
Based on the results, we develop a two-phase framework for constructing a discrimination-free classifier with a theoretical guarantee. The experiments demonstrate the theoretical results and show the effectiveness of our two-phase framework.
\end{abstract}

\section{Introduction}
Discrimination-aware classification is receiving an increasing attention in the data mining and machine learning fields. Many methods have been proposed for constructing discrimination-free classifiers, which can be broadly classified into three categories: the pre-process methods that modify the training data \cite{kamiran2009classifying,feldman2015certifying,zhang2017causal,calmon2017optimized,nabi2017fair}, the in-process methods that adjust the learning process of the classifier \cite{calders2010three,kamishima2011fairness,kamishima2012fairness,zafar2017fairness}, and the post-process methods that directly change the predicted labels \cite{kamiran2010discrimination,hardt2016equality}. All three categories of methods have their respective limitations. For the in-process methods, they usually perform some tweak or develop some regularizers for the classifier to correct or penalize discriminatory outcomes during the learning process. However, since the discrimination or fair constraints are generally not convex functions, surrogate functions are usually used for the minimization. For example, in \cite{zafar2017fairness}, the covariance between the protected attribute and the signed distance from the data point to the decision boundary is used as the surrogate function. As a result, additional bias might be introduced due to the approximation errors associated with the surrogate function. For the post-process methods, they are restricted to those who can modify the predicted label of each individual independently. For example, in \cite{hardt2016equality}, the probability of changing the predicted label to one given an individual's protected attribute is optimized. Thus, methods that map the whole dataset or population to a new non-discriminatory one cannot be adopted for post-process, which means that a number of causal-based discrimination removal methods (e.g., \cite{zhang2017causal,nabi2017fair}) cannot be applied.

In our work, we target the pre-process methods that modify the training data, where some methods only modify the label, such as the \emph{Massaging} \cite{kamiran2009classifying,zliobaite2011handling} and the \emph{Causal-Based Removal} \cite{zhang2017causal}, and some methods also modify the data attributes other than the label, such as the \emph{Preferential Sampling} \cite{kamiran2012data,zliobaite2011handling}, the \emph{Reweighing} \cite{calders2009building}, and the \emph{Disparate Impact Removal} \cite{feldman2015certifying,adler2016auditing}. 
%The limitation of the pre-process methods is that there is no guarantee about the discrimination in prediction. 
The fundamental assumption of the pre-process methods is that, since the classifier is learned from a discrimination-free dataset, it is likely that the future prediction will also be discrimination-free \cite{kamiran2009discrimination}. 
Although this assumption is plausible, however, there is no theoretical guarantee to show ``how much likely'' it is and ``how discrimination-free'' the prediction would be given a training data and a classifier. The lack of the theoretical guarantees places great uncertainty on the performance of all pre-process methods.

In this paper, we fill this gap by modeling the discrimination in prediction using the causal model. A causal model \cite{pearl2009causality} is a structural equation-based mathematical object that describes the causal mechanism of a system.  We assume that there exists a fixed but unknown causal model that represents the underlying data generation mechanism of the population. Based on the causal model, we define the causal measure of discrimination in population as well as in prediction. We then formalize two problems, namely discovering and removing discrimination in prediction. Based on specific assumptions regarding the causal model and the causal measure of discrimination, we conduct concrete analysis of discrimination. We derive the formula for quantitatively measuring the discriminatory effect in population from the observable probability distributions. We then derive the corresponding causal measure of the discrimination in prediction, as well as their approximations from the sample dataset. Finally, we link the discrimination in prediction with the discrimination in the training data by a probabilistic condition, which mathematically bounds the probability of the discrimination in prediction being within a given interval in terms of the training data and classifier.

As a consequence, we obtain two important theoretical results: (1) even if the discrimination in the training data is completely removed, the discrimination in prediction can still exist due to the bias in the classifier; and (2) for removing discrimination, different from the claims of many previous work, not all pre-process methods can ensure non-discrimination in prediction even though they can achieve non-discrimination in the modified training data. We show that to guarantee non-discrimination in prediction, the pre-process methods should only modify the label.
%only the methods that only modify the label can guarantee non-discrimination in prediction. 
Based on the results, we develop a two-phase framework for constructing a discrimination-free classifier with a theoretical guarantee, which provides a guideline for employing existing pre-process methods or designing new ones. The experiments demonstrate the theoretical results and show the effectiveness of our two-phase framework.

\section{Problem Formulation}

\subsection{Notations and Preliminaries}
We consider an attribute space which consists of some protected attributes, the label of certain decision attribute, and the non-protected attributes. Throughout the paper, we use an uppercase alphabet, e.g., $X$ to represent an attribute; a bold uppercase alphabet, e.g., $\mathbf{X}$, to represent a subset of attributes. We use a lowercase alphabet, e.g., $x$, to represent a realization or instantiation of attribute $X$; a bold lowercase alphabet, e.g., $\mathbf{x}$, to represent a realization or instantiation of $\mathbf{X}$. For ease of representation, we assume that there is only one protected attribute, denoted by $C$, which is a binary attribute associated with the domain values of the non-protected group $c^{+}$ and the protected group $c^{-}$. We denote the label of the decision attribute by $L$, which is a binary attribute associated with the domain values of the positive label $l^{+}$ and negative label $l^{-}$. According to the convention in machine learning, we also define that $l^{+}=1$ and $l^{-}=0$. The set of all the non-protected attributes is denoted by $\mathbf{Z}=\{Z_{1},\cdots,Z_{m}\}$. %Please refer to the notation table shown as Table \ref{tab:no}.

%\begin{table}[tbp]\small
%\centering
%\caption{Table of notations.}
%\label{tab:no}
%\begin{tabular}{ll}
%\Xhline{0.75pt}
%Notation		&		Definition	\\
%\cline{1-2}
%$C$             & Protected attribute                                          \\
%$\mathbf{Z}=\{Z_{1},\cdots,Z_{m}\}$ & Non-protected attributes                           \\
%$L$						& Label	of decision																													\\
%$h:C\times \mathbf{Z} \rightarrow L$ 	&	Classifier \\
%$\mathcal{M}$		&		Causal model of population \\
%$\mathcal{M}_{h}$		&		Causal model of prediction \\
%$\mathcal{D}=\{(c^{(j)},\mathbf{z}^{(j)}, l^{(j)})\}$		&		Training data \\
%$\mathcal{D}_{h}=\{(c^{(j)},\mathbf{z}^{(j)}, h(c^{(j)}\!,\!\mathbf{z}^{(j)}))\}$ & Training data w/ predicted labels \\
%\Xhline{0.75pt}
%\end{tabular}
%\end{table}

%\subsection{Causal Model}
%A causal model is a mathematical object that describes the data generation mechanisms of a system as a set of structural equations. It is formally defined as follows.

A causal model is formally defined as follows.

%\begin{definition}[Causal Model]\label{def:cm}
%A causal model $\mathcal{M}$ is a triple $\mathcal{M} = \langle \mathbf{U},\mathbf{V},\mathbf{F} \rangle$ where $\mathbf{U}$ is a set of hidden contextual variables that are determined by factors outside the model, $\mathbf{V}$ is a set of observed variables that are determined by variables in $\mathbf{U}\cup\mathbf{V}$, and $\mathbf{F}$ is a set of equations mapping from $\mathbf{U}\times \mathbf{V}$ to $\mathbf{V}$. Specifically, for each $V_{i}\in \mathbf{V}$, there is an equation $f_{i}$ mapping from $\mathbf{U}\times (\mathbf{V}\backslash V_{i})$ to $V_{i}$, i.e., 
	%\begin{equation*}
	%v_{i} = f_{i}(pa_{i},\mathbf{u}_{i}),
	%\end{equation*}
	%where $pa_{i}$ is a realization of a set of observed variables $PA_{i}\subseteq \mathbf{V}\backslash V_{i}$ called the parents of $V_{i}$, and $\mathbf{u}_{i}$ is a realization of a set of hidden variables $\mathbf{U}_{i}\subseteq \mathbf{U}$.
%\end{definition}

\begin{definition}[Causal Model]\label{def:cm}
A causal model $\mathcal{M}$ is a triple $\mathcal{M} = \langle \mathbf{U},\mathbf{V},\mathbf{F} \rangle$ where
\begin{enumerate}
	\item $\mathbf{U}$ is a set of hidden contextual variables that are determined by factors outside the model. %A joint probability distribution $P(\mathbf{u})$ is defined over $\mathbf{U}$.
	\item $\mathbf{V}$ is a set of observed variables that are determined by variables in $\mathbf{U}\cup\mathbf{V}$.
	\item $\mathbf{F}$ is a set of equations mapping from $\mathbf{U}\times \mathbf{V}$ to $\mathbf{V}$. Specifically, for each $V_{i}\in \mathbf{V}$, there is an equation $f_{i}$ mapping from $\mathbf{U}\times (\mathbf{V}\backslash V_{i})$ to $V_{i}$, i.e., 
	\begin{equation*}
	v_{i} = f_{i}(pa_{i},\mathbf{u}_{i}),
	\end{equation*}
	where $pa_{i}$ is a realization of a set of observed variables $PA_{i}\subseteq \mathbf{V}\backslash V_{i}$ called the parents of $V_{i}$, and $\mathbf{u}_{i}$ is a realization of a set of hidden variables $\mathbf{U}_{i}\subseteq \mathbf{U}$.
\end{enumerate}
\end{definition}

The causal effect in the causal model is defined over the intervention that fixes the value of an observed variable(s) $V$ to a constant(s) $v$ while keeping the rest of the model unchanged. The intervention is mathematically formalized as $do(V=v)$ or simply $do(v)$. Then, for any variables $X,Y\in \mathbf{V}$, the distribution of $Y$ after $do(x)$ is defined as \cite{pearl2009causality}
\begin{equation}\label{eq:pyx}
P(y|do(x)) \triangleq P(Y=y|do(X=x)) = \sum_{\{\mathbf{u}:Y_{x}(\mathbf{u})=y\}}P(\mathbf{u}),
\end{equation}
where $Y_{x}(\mathbf{u})$ denotes the value of $Y$ after intervention $do(x)$ under context $\mathbf{U}=\mathbf{u}$.

Note that $P(\mathbf{u})$ is an unknown joint distribution of all hidden variables. If the causal model satisfies the Markovian assumption: (1) the associated causal graph of the causal model is acyclic; and (2) all variables in $\mathbf{U}$ are mutually independent, $P(y|do(x))$ can be computed from the joint distribution of $\mathbf{V}$ according to the truncated factorization formula \cite{pearl2009causality}
\begin{equation}\label{eq:tff}
P(y|do(x)) = \sum_{\mathbf{v}'} \prod_{V_{i}\neq X}P(v_{i}|pa_{i})_{\delta_{X=x}},
\end{equation}
where the summation is a marginalization that traverses all value combinations of $\mathbf{V}'=\mathbf{V}\backslash \{X,Y\}$, and $\delta_{X=x}$ means replacing $X$ with $x$ in each term.

%Each model $\mathcal{M}$ is associated with a direct graph $\mathcal{G}(\mathcal{M})$, where each node in the graph corresponds to a variable $X_{i}$ in $\mathbf{V}$, and direct edges point from each member of $PA_{i}$ toward $X_{i}$. Such graph is called the causal graph associated with $\mathcal{M}$.

\subsection{Model Discrimination in Prediction}

Assume that there exists a fixed population over the space $C\times \mathbf{Z}\times L$. The values of all the attributes in the population are determined by a causal model $\mathcal{M}$, which can be written as
\begin{equation*}
\textrm{Causal Model $\mathcal{M}$} \quad\quad
\begin{array}{l}
c = f_{C}(pa_{C},\mathbf{u}_{C}) \\
z_{i} = f_{i}(pa_{i},\mathbf{u}_{i}) \quad i=1,\cdots,m \\
l = f_{L}(pa_{L},\mathbf{u}_{L})
\end{array}
\end{equation*}
where $f_{L}$ can be considered as the decision making process in the real system.
Without ambiguity, we can also use $\mathcal{M}$ to denote the population, and use the terms mechanism and population interchangeably. In practice, $\mathcal{M}$ is unknown and we can only observe a sample dataset $\mathcal{D}=\{(c^{(j)},\mathbf{z}^{(j)}, l^{(j)});j=1,\cdots,n\}$ drawn from the population.

A classifier $h$ is function mapping from $C\times \mathbf{Z}$ to $L$. A hypothesis space $\mathcal{H}$ is a set of candidate classifiers. A learning algorithm analyzes dataset $\mathcal{D}$ as the training data to find a classifier from $\mathcal{H}$ that minimizes the difference between the predicted labels $h(c^{(j)},\mathbf{z}^{(j)})$ and the true labels $l^{(j)}$ ($j=1,\cdots,n$). Once training completes, the classifier is deployed to infer prediction on any new unlabeled data. %i.e., the classifier computes the predicted label for any unlabeled individual. 
It is usually assumed that the unlabeled data is drawn from the same population as the training data, i.e., $\mathcal{M}$. Therefore, in prediction, the values of all the attributes other than the label are determined by the same mechanisms as those in $\mathcal{M}$, and meanwhile the classifier acts as a new mechanism for determining the value of the label. We consider $\mathcal{M}$ with function $f_{L}(\cdot)$ being replaced by classifier $h(\cdot)$ as the causal model of classifier $h$, denoted by $\mathcal{M}_{h}$, which is written as
\begin{equation*}
\textrm{Causal Model $\mathcal{M}_{h}$} \quad\quad
\begin{array}{l}
c = f_{C}(pa_{C},\mathbf{u}_{C}) \\
z_{i} = f_{i}(pa_{i},\mathbf{u}_{i}) \quad i=1,\cdots,m \\
l = h(c,\mathbf{z})
\end{array}
\end{equation*}
If we apply the classifier on $\mathcal{D}$, we can obtain a new dataset $\mathcal{D}_{h}$ by replacing the original labels with the predicted labels, i.e., $\mathcal{D}_{h}=\{(c^{(j)},\mathbf{z}^{(j)}, h(c^{(j)},\mathbf{z}^{(j)}));j=1,\cdots,n\}$. Straightforwardly, $\mathcal{D}_{h}$ can be considered as a sample drawn from $\mathcal{M}_{h}$.

The discrimination in prediction made by classifier $h$ can be naturally defined as the discrimination in $\mathcal{M}_{h}$. To this end, we first define a measure of discrimination in $\mathcal{M}$ based on the causal relationship specified by $\mathcal{M}$, denoted by $\mathrm{DE}_{\mathcal{M}}$ called the \emph{true discrimination}. By adopting the same measure, we denote the discrimination in $\mathcal{M}_{h}$ by $\mathrm{DE}_{\mathcal{M}_{h}}$, called the \emph{true discrimination in prediction}. Then, we denote the approximation of $\mathrm{DE}_{\mathcal{M}}$ from dataset $\mathcal{D}$ by $\mathrm{DE}_{\mathcal{D}}$, and similarly denote the approximation of $\mathrm{DE}_{\mathcal{M}_{h}}$ from dataset $\mathcal{D}_{h}$ by $\mathrm{DE}_{\mathcal{D}_{h}}$.

%The discrimination in prediction made by classifier $h$ can be naturally defined as the discrimination in $\mathcal{M}_{h}$. To this end, we first define a measure of discrimination in $\mathcal{M}$ based on the causal relationship specified by $\mathcal{M}$, denoted by $\mathrm{DE}_{\mathcal{M}}$. By adopting the same measure, we denote the discrimination in $\mathcal{M}_{h}$ by $\mathrm{DE}_{\mathcal{M}_{h}}$. Then, we denote the approximation of $\mathrm{DE}_{\mathcal{M}}$ from dataset $\mathcal{D}$ by $\mathrm{DE}_{\mathcal{D}}$, and similarly denote the approximation of $\mathrm{DE}_{\mathcal{M}_{h}}$ from dataset $\mathcal{D}_{h}$ by $\mathrm{DE}_{\mathcal{D}_{h}}$. 

Our target is to discover and remove the true discrimination in prediction, i.e., $\mathrm{DE}_{\mathcal{M}_{h}}$, based on certain causal measure of discrimination defined on $\mathcal{M}$, i.e., $\mathrm{DE}_{\mathcal{M}}$. When calculating $\mathrm{DE}_{\mathcal{M}_{h}}$ from dataset $\mathcal{D}$, we may encounter disturbances such as the sampling error of $\mathcal{D}$ and the misclassification of $h$. We then need to compute analytic approximation to $\mathrm{DE}_{\mathcal{M}_{h}}$. Thus, we define the problem of discovering discrimination in prediction as follows.

\begin{problem}[Discover Discrimination in Prediction]
Given a causal measure of discrimination defined on $\mathcal{M}$, i.e., $\mathrm{DE}_{\mathcal{M}}$, a sample dataset $\mathcal{D}$ and a classifier $h$ trained on $\mathcal{D}$, compute analytic approximation to the true discrimination in prediction, i.e., $\mathrm{DE}_{\mathcal{M}_{h}}$. %by bounding disturbances caused by sampling error and misclassification.
\end{problem}

If the true discrimination in prediction is detected according to the approximation, the next step is to remove the discrimination through tweaking the dataset and/or the classifier. Thus, we define the problem of removing discrimination in prediction as follows.

\begin{problem}[Remove Discrimination in Prediction]
Given $\mathrm{DE}_{\mathcal{M}}$, $\mathcal{D}$ and $h$, tweak $\mathcal{D}$ and/or $h$ in order to make $\mathrm{DE}_{\mathcal{M}_{h}}$ be bounded by a user-defined threshold $\tau$.
\end{problem}

\section{Discover Discrimination in Prediction}
In the above general problem definitions, $\mathrm{DE}_{\mathcal{M}}$ can be any reasonable causal measure of discrimination defined on any causal model. However, a concrete analysis of discrimination must rely on specific assumptions regarding the causal measure of discrimination and the causal model. The remaining of the paper is based on following assumptions: (1) $\mathcal{M}$ satisfies the Markovian assumption; (2) we consider all causal effects (total effect) of $C$ on $L$ as discriminatory; (3) we assume that $C$ has no parent and $L$ has no child. The first assumption is necessary for computing the causal effect from the observable probability distributions. The second assumption is because the total causal effect is the causal relationship that is easiest to interpret and estimate. We will extend our results to other discrimination definitions such as those in \cite{zhang2017causal,bonchi2017exposing} in the future work. The last assumption is to make our theoretical results more concise and can be easily relaxed.

\subsection{Causal Measure of Discrimination}
We first derive the true discrimination in $\mathcal{M}$. The key of discrimination is a counterfactual question: whether the decision of an individual would be different had the individual been of a different protected/non-protected group (e.g., sex, race, age, religion, etc.)? To answer this question, we can perform an intervention on each individual to change the value of protected attribute $C$ and see how label $L$ will change. We consider the difference between the expectation of the labels when performing $do(c^{+})$ for all individuals and the expectation of the labels when performing $do(c^{-})$ for all individuals, and use it as the causal measure of discrimination.

To obtain the above difference, note that the causal model is completely specified at the individual level when context $\mathbf{U}=\mathbf{u}$ is specified. For each individual specified by $\mathbf{u}$, denote the label of individual $\mathbf{u}$ by $L_{c^{+}}(\mathbf{u})$ (resp. $L_{c^{-}}(\mathbf{u})$) when $C$ is fixed according to $do(c^{+})$ (resp. $do(c^{-})$). Then, the difference in the label of individual $\mathbf{u}$ is given by $L_{c^{+}}(\mathbf{u}) - L_{c^{-}}(\mathbf{u})$. The expected difference of the labels over all individuals is hence given by $\mathbb{E}[ L_{c^{+}}(\mathbf{u}) - L_{c^{-}}(\mathbf{u}) ]$. 
Based on this analysis, we obtain the following proposition.

\begin{proposition}\label{thm:dem}
Given a causal model $\mathcal{M}$, the true discrimination is given by
\begin{equation*}
\mathrm{DE}_{\mathcal{M}} = P(l^{+}|c^{+}) - P(l^{+}|c^{-}).
\end{equation*}
\end{proposition}

\begin{proof}
The above expectations can be computed as
\begin{equation}\label{eq:lc}
\begin{split}
& \mathbb{E}[L_{c^{+}}(\mathbf{u})] = \sum_{\mathbf{u}}L_{c^{+}}(\mathbf{u})P(\mathbf{u}) = \sum_{\{\mathbf{u}:L_{c^{+}}(\mathbf{u})=l^{+}\}} l^{+} P(\mathbf{u}) \\
& ~ + \sum_{\{\mathbf{u}:L_{c^{+}}(\mathbf{u})=l^{-}\}} l^{-} P(\mathbf{u}) = \sum_{\{\mathbf{u}:L_{c^{+}}(\mathbf{u})=l^{+}\}} P(\mathbf{u}) = P(l^{+}|do(c^{+})),
\end{split}
\end{equation}
where the last step is according to Eq. \eqref{eq:pyx}. According to Eq. \eqref{eq:tff}, we have
\begin{equation*}
P(l^{+}|do(c^{+})) = \sum_{\mathbf{z}} P(l^{+}|pa_{L})_{\delta_{C=c^{+}}} \prod_{Z_{i}\in \mathbf{Z}}P(z_{i}|pa_{i})_{\delta_{C=c^{+}}}.
\end{equation*}
On the other hand, since $C$ has no parent, we have
\begin{equation*}
P(l^{+}|c^{+}) = \frac{P(l^{+},c^{+})}{P(c^{+})} = \frac{\sum_{\mathbf{z}}P(c^{+})P(l^{+}|pa_{L})\prod_{Z_{i}\in \mathbf{Z}}P(z_{i}|pa_{i})}{P(c^{+})},
\end{equation*}
which can be rewritten as
\begin{equation*}
\sum_{\mathbf{z}} P(l^{+}|pa_{L})_{\delta_{C=c^{+}}} \prod_{Z_{i}\in \mathbf{Z}}P(z_{i}|pa_{i})_{\delta_{C=c^{+}}}.
\end{equation*}
Thus, we have $P(l^{+}|do(c^{+})) = P(l^{+}|c^{+})$, leading to $\mathbb{E}[L_{c^{+}}(\mathbf{u})] = P(l^{+}|c^{+})$. Similarly we can prove $\mathbb{E}[L_{c^{-}}(\mathbf{u})] = P(l^{+}|c^{-})$. Hence, the proposition is proven.
\end{proof}

Interestingly, the obtained discrimination causal measure is the same as the classic statistical discrimination metric \emph{risk difference}, which is widely used as the non-discrimination constraint in discrimination-aware learning \cite{romei2014multidisciplinary}. Our analysis can help understand the assumptions and scenarios in which the risk difference applies.

Given dataset $\mathcal{D}$, we approximate $\mathrm{DE}_{\mathcal{M}}$ using the maximum likelihood estimation, denoted by $\mathrm{DE}_{\mathcal{D}}$ as shown below.

%we define $\mathrm{DE}_{\mathcal{D}}$ as the maximum likelihood estimation of $\mathrm{DE}_{\mathcal{M}}$, as shown below.

\begin{proposition}\label{thm:define_ded}
Given a dataset $\mathcal{D}$, the discrimination in $\mathcal{D}$ is given by
\begin{equation*}
\mathrm{DE}(c^{+},c^{-})_{\mathcal{D}} = \hat{P}(l^{+}|c^{+}) - \hat{P}(l^{+}|c^{-}),
\end{equation*}
where
\begin{equation}\label{eq:plc}
\hat{P}(l^{+}|c^{+}) = \sum_{\mathbf{z}}\hat{P}(l^{+}|c^{+},\mathbf{z})\hat{P}(\mathbf{z}|c^{+}),
\end{equation}
with $\hat{P}(\cdot)$ being the conditional frequency in $\mathcal{D}$.
%with $n^{+}$ (resp. $n^{-}$) is the number of individuals in $\mathcal{D}$ having $C=c^{+}$ (resp. $C=c^{-}$).
\end{proposition}

Given a classifier $h: C\times \mathbf{Z} \rightarrow L$, denote the predicted labels by $\tilde{L}$. By adopting the same causal measure of discrimination of $\mathrm{DE}_{\mathcal{M}}$, we obtain $\mathrm{DE}_{\mathcal{M}_{h}}$ shown as follows.

\begin{proposition}\label{thm:compute_mh}
Given a causal model $\mathcal{M}$ and a classifier $h$, the true discrimination in prediction is given by
\begin{equation*}
\mathrm{DE}_{\mathcal{M}_{h}} = P(\tilde{l}^{+}|c^{+}) - P(\tilde{l}^{+}|c^{-}),
\end{equation*}
where $P(\tilde{l}^{+}|c^{+})$ (resp. $P(\tilde{l}^{+}|c^{-})$) is the probability of the positive predicted labels for the data with $C=c^{+}$ (resp. $C=c^{-}$). %given by
%\begin{equation}
%P(\tilde{l}^{+}|c^{+}) = \sum_{\mathbf{z}}\mathbbm{1}_{[h(c^{+},\mathbf{z})=l^{+}]}P(\mathbf{z}|c^{+}).
%\end{equation}
\end{proposition} 

We similarly define $\mathrm{DE}_{\mathcal{D}_{h}}$ as the maximum likelihood estimation of $\mathrm{DE}_{\mathcal{M}_{h}}$.

\begin{proposition}
Given a dataset $\mathcal{D}$ and a classifier $h$ trained on $\mathcal{D}$, the discrimination in $\mathcal{D}_{h}$ is given by
\begin{equation*}
\mathrm{DE}_{\mathcal{D}_{h}} = \hat{P}(\tilde{l}^{+}|c^{+}) - \hat{P}(\tilde{l}^{+}|c^{-}),
\end{equation*}
where
\begin{equation}\label{eq:plc_}
\hat{P}(\tilde{l}^{+}|c^{+}) = \sum_{\mathbf{z}}\mathbbm{1}_{[h(c^{+},\mathbf{z})=l^{+}]}\hat{P}(\mathbf{z}|c^{+}). %= \frac{1}{n^{+}} \sum_{\{j:c^{(j)}=c^{+}\}}\mathbbm{1}_{[h(c^{(j)},\mathbf{z}^{(j)})=l^{+}]}.
\end{equation}
%Here $n^{+}$ and $n^{-}$ ($n^{+}+n^{-}=n$) are the numbers of individuals with $c^{+}$ and $c^{-}$ in $\mathcal{D}$.
%with $n^{+}$ (resp. $n^{-}$) is the number of individuals in $\mathcal{D}$ having $C=c^{+}$ (resp. $C=c^{-}$).
\end{proposition}

\subsection{Bounding Discrimination in prediction}
To approximate $\mathrm{DE}_{\mathcal{M}_{h}}$ from $\mathcal{D}$, sampling error cannot be avoided since $\mathcal{D}$ is only a subset of the entire population. Although exact measurement of sampling error is generally not feasible as $M$ is unknown, it can be probabilistically bounded. In the following we first bound the difference between $\mathrm{DE}_{\mathcal{M}}$ and its approximation $\mathrm{DE}_{\mathcal{D}}$, and then extend the result to the difference between $\mathrm{DE}_{\mathcal{M}_{h}}$ and its approximation $\mathrm{DE}_{\mathcal{D}_{h}}$.

%To estimate $\mathrm{DE}_{\mathcal{M}}$ using $\mathrm{DE}_{\mathcal{D}}$, we bound the difference between $\mathrm{DE}_{\mathcal{M}}$ and $\mathrm{DE}_{\mathcal{D}}$ for any causal model $\mathcal{M}$ and sampled/generated dataset $\mathcal{D}$ in term of the sample size of $\mathcal{D}$, as shown in the following.

%To estimate $\mathrm{DE}_{\mathcal{M}}$ using $\mathrm{DE}_{\mathcal{D}}$, we bound the distance between $\mathrm{DE}_{\mathcal{M}}$ and $\mathrm{DE}_{\mathcal{D}}$ in term of the sample size of $\mathcal{D}$.

\begin{proposition}\label{thm:md}
Given a causal model $\mathcal{M}$ and a sample dataset $\mathcal{D}$ with size of $n$, the probability of the difference between $\mathrm{DE}_{\mathcal{M}}$ and $\mathrm{DE}_{\mathcal{D}}$ no larger than $t$ is bounded by
\begin{equation*}
P~ \bigg( \left| \mathrm{DE}_{\mathcal{M}} - \mathrm{DE}_{\mathcal{D}} \right| \leq t \bigg) > 1-4e^{-\frac{n^{+}n^{-}}{n}t^{2}},
\end{equation*}
where $n^{+}$ and $n^{-}$ ($n^{+}+n^{-}=n$) are the numbers of individuals with $c^{+}$ and $c^{-}$ in $\mathcal{D}$.
\end{proposition}

\begin{proof}
By definition of $\mathrm{DE}_{\mathcal{M}}$ and $\mathrm{DE}_{\mathcal{D}}$ we have
\begin{equation*}
\begin{split}
\mathrm{DE}_{\mathcal{M}} - \mathrm{DE}_{\mathcal{D}} %& = P(l^{+}|c^{+}) - P(l^{+}|c^{-}) - \Big( \hat{P}(l^{+}|c^{+}) - \hat{P}(l^{+}|c^{-}) \Big) \\
& = \Big( P(l^{+}|c^{+})-\hat{P}(l^{+}|c^{+}) \Big) + \Big( \hat{P}(l^{+}|c^{-})-P(l^{+}|c^{-}) \Big).
\end{split}
\end{equation*}

Denoting by $l^{(+j)}$ the label of the $j$th individual in $\mathcal{D}$ with $C=c^{+}$, we can write $\hat{P}(l^{+}|c^{+})$ as 
\begin{equation*}
\hat{P}(l^{+}|c^{+}) = \frac{1}{n^{+}}\Big( \mathbbm{1}_{[l^{(+1)}=l^{+}]}+\cdots+\mathbbm{1}_{[l^{(+n^{+})}=l^{+}]} \Big),
\end{equation*}
where indicators $\mathbbm{1}_{[l^{(+j)}=l^{+}]}$ ($j=1\cdots n^{+}$) can be considered as independent random variables bounded by the interval $[0, 1]$. Note that $\mathbb{E}[\hat{P}(l^{+}|c^{+})] = P(l^{+}|c^{+})$. According to the Hoeffding's inequality \cite{murphy2012machine}, we have
\begin{equation*}
P~ \bigg( \left| P(l^{+}|c^{+}) - \hat{P}(l^{+}|c^{+}) \right| \geq t \bigg) \leq 2e^{-2n^{+}t^{2}}.
\end{equation*}
Similarly, we have
%\begin{equation*}
$P~ \left( \left| P(l^{+}|c^{-}) - \hat{P}(l^{+}|c^{-}) \right| \geq t \right) \leq 2e^{-2n^{-}t^{2}}$.
%\end{equation*}
Therefore, we have
\begin{equation}\label{eq:md}
\begin{split}
& P~ \bigg( \left| \mathrm{DE}_{\mathcal{M}} - \mathrm{DE}_{\mathcal{D}} \right| \leq t \bigg) \\
\geq & P~ \bigg( \left| P(l^{+}|c^{+}) - \hat{P}(l^{+}|c^{+}) \right| + \left| P(l^{+}|c^{-}) - \hat{P}(l^{+}|c^{-}) \right| \leq t \bigg) \\
%\geq & \max_{0\leq x \leq t} P \bigg( \left| P(l^{+}|c^{+}) \!-\! \hat{P}(l^{+}|c^{+}) \right| \!\leq\! x \wedge \left| P(l^{+}|c^{-}) \!-\! \hat{P}(l^{+}|c^{-}) \right| \!\leq\! t\!-\!x \bigg) \\
\geq & \max_{0\leq x \leq t} P \bigg( \left| P(l^{+}|c^{+}) \!-\! \hat{P}(l^{+}|c^{+}) \right| \!\leq\! x \bigg) P \bigg( \left| P(l^{+}|c^{-}) \!-\! \hat{P}(l^{+}|c^{-}) \right| \!\leq\! t\!-\!x \bigg) \\
\geq & \max_{0\leq x \leq t} (1-2e^{-2n^{+}x^{2}})(1-2e^{-2n^{-}(t-x)^{2}}) \\
> & 1-4e^{-\frac{n^{+}n^{-}}{n}t^{2}},
\end{split}
\end{equation}
where the last line is by substituting $x$ with $\frac{\sqrt{n^-}}{\sqrt{n^{+}}+\sqrt{n^{-}}}t$.

%\begin{equation}\label{eq:md}
%\begin{split}
%& P~ \bigg( \left| \mathrm{DE}_{\mathcal{M}} - \mathrm{DE}_{\mathcal{D}} \right| \geq t \bigg) \\
%\leq & P~ \bigg( \left| P(l^{+}|c^{+}) - \hat{P}(l^{+}|c^{+}) \right| + \left| P(l^{+}|c^{-}) - \hat{P}(l^{+}|c^{-}) \right| \geq t \bigg) \\
%%= & \int_{0}^{t}\!\!\! P \bigg( \left| P(l^{+}|c^{+}) \!-\! \hat{P}(l^{+}|c^{+}) \right| \!\geq\! x \wedge \left| P(l^{+}|c^{-}) \!-\! \hat{P}(l^{+}|c^{-}) \right| \!\geq\! t\!-\!x \bigg) \mathrm{d}x \\
%= & \int_{0}^{t}\!\!\! P \bigg( \left| P(l^{+}|c^{+}) \!-\! \hat{P}(l^{+}|c^{+}) \right| \!\geq\! x \bigg) P \bigg( \left| P(l^{+}|c^{-}) \!-\! \hat{P}(l^{+}|c^{-}) \right| \!\geq\! t\!-\!x \bigg) \mathrm{d}x \\
%\leq & \int_{0}^{1} 4e^{-2n^{+}x^{2}}e^{-2n^{-}(t-x)^{2}}\mathrm{d}x \leq \sqrt{\frac{8\pi}{n}}e^{-\frac{2n^{+}n^{-}}{n}t^{2}}.
%%= & \sqrt{\frac{2\pi}{n}}e^{-\frac{2n^{+}n^{-}t^2}{n}}\left(\operatorname{erf}\left(\frac{\sqrt{2}n^{-}t}{\sqrt{n}}\right)-\operatorname{erf}\left(\frac{\sqrt{2}\left(n^{-}\left(t-1\right)-n^{+}\right)}{\sqrt{n}}\right)\right) \\
%%\leq & \sqrt{\frac{8\pi}{n}}e^{-\frac{2n^{+}n^{-}}{n}t^{2}}
%\end{split}
%\end{equation}
%The fourth line of the above expression is due to that each individual is independently drawn from the population. %Hence, the theorem is proven.
\end{proof}

%We bound the difference between $\mathrm{DE}_{\mathcal{M}_{h}}$ and $\mathrm{DE}_{\mathcal{D}_{h}}$ for any causal model $\mathcal{M}$, dataset $\mathcal{D}$ and classifier $h\in \mathcal{H}$ in terms of the sample size of $\mathcal{D}$ and the size of the hypothesis space $\mathcal{H}$. 

For extending Proposition \ref{thm:md} to Proposition \ref{thm:mhdh}, the difference is that, since $h$ is a classifier depending on training data $\mathcal{D}$, indicators $\mathbbm{1}_{[h(c^{(+j)},\mathbf{z}^{(+j)})=l^{+}]}$ cannot be considered as independent. Thus, the Hoeffding's inequality cannot be directly applied and a uniform bound for all hypotheses in $\mathcal{H}$ is needed.

\begin{proposition}\label{thm:mhdh}
Given a causal model $\mathcal{M}$, a sample dataset $\mathcal{D}$, and a classifier $h: C\times \mathbf{Z} \rightarrow L$ from hypothesis space $\mathcal{H}$, the probability of the distance between $\mathrm{DE}_{\mathcal{M}_{h}}$ and $\mathrm{DE}_{\mathcal{D}_{h}}$ no larger than $t$ is bounded by
\begin{equation*}
P~ \Bigg(  \left| \mathrm{DE}_{\mathcal{M}_{h}} - \mathrm{DE}_{\mathcal{D}_{h}} \right| \leq t \Bigg) \geq 1-\delta(t),
\end{equation*}
where
\begin{equation*}
\delta(t) = 
\begin{cases}
\displaystyle 4|\mathcal{H}|^{2} e^{-\frac{n^{+}n^{-}}{n}t^{2}} & \textrm{if $\mathcal{H}$ is finite,}\\
\displaystyle 4 \frac{(2en^{+})^{d}+(2en^{-})^{d}}{d^d} e^{-\frac{n^{+}n^{-}}{n}t^{2}} & \textrm{if $\mathcal{H}$ is infinite,}
\end{cases}
\end{equation*}
with $d$ being the VC dimension of $\mathcal{H}$.
%\begin{equation*}
%\delta(t) = |\mathcal{H}|^{2} \sqrt{\frac{8\pi}{n}} e^{-\frac{2n^{+}n^{-}}{n}t^{2}}
%\end{equation*}
%if $\mathcal{H}$ is finite with the size of $|\mathcal{H}|$, and
%\begin{equation*}
%\delta(t) = \left( \frac{4e^{2}n^{+}n^{-}}{d^{2}} \right)^{d}\sqrt{\frac{128\pi}{n}} e^{-\frac{2n^{+}n^{-}}{n}t^{2}}
%\end{equation*}
%if $\mathcal{H}$ is infinite with the VC dimension of $d$.
\end{proposition}

\begin{proof}
According to the definitions of $\mathrm{DE}_{\mathcal{M}_{h}}$ and $\mathrm{DE}_{\mathcal{D}_{h}}$,
\begin{equation*}
\mathrm{DE}_{\mathcal{M}_{h}} - \mathrm{DE}_{\mathcal{D}_{h}} = \left( P(\tilde{l}^{+}|c^{+})-\hat{P}(\tilde{l}^{+}|c^{+}) \right) + \left( P(\tilde{l}^{+}|c^{-})-\hat{P}(\tilde{l}^{+}|c^{-}) \right).
\end{equation*}
Similar to the proof of Proposition \ref{thm:md}, we treat $\hat{P}(\tilde{l}^{+}|c^{+})$ as the mean of indicators $\mathbbm{1}_{[h(c^{(+j)},\mathbf{z}^{(+j)})=l^{+}]}$ ($j=1\cdots n^{+}$). According to the generalization bound in statistical learning theory \cite{vapnik1998statistical}, if $\mathcal{H}$ is finite we have
%Note that since $h$ is a classifier depending on the training data $\mathcal{D}$, these random variables are not independent. Thus, the Hoeffding's inequality cannot be directly applied and we need a uniform bound for all hypotheses in $\mathcal{H}$. According to the studies in statistical learning theory \cite{vapnik1998statistical}, if $\mathcal{H}$ is finite we have
\begin{equation*}
P\left( \left| P(\tilde{l}^{+}|c^{+})-\hat{P}(\tilde{l}^{+}|c^{+}) \right| \geq t \right) ~\leq~ 2|\mathcal{H}| e^{-2n^{+}t^{2}},
\end{equation*}
where $|\mathcal{H}|$ is the size of $\mathcal{H}$. If $\mathcal{H}$ is infinite we have
\begin{equation*}
P\left( \left| P(\tilde{l}^{+}|c^{+})-\hat{P}(\tilde{l}^{+}|c^{+}) \right| \geq t \right) ~\leq~ 4\left( \frac{2en^{+}}{d} \right)^{d} e^{-2n^{+}t^{2}},
\end{equation*}
where $d$ is the VC dimension of $\mathcal{H}$. Then the proposition can be proven similarly to Eq. \eqref{eq:md}.
\end{proof}

Proposition \ref{thm:mhdh} provides an approximation of $\mathrm{DE}_{\mathcal{M}_{h}}$ from $\mathrm{DE}_{\mathcal{D}_{h}}$. However, since pre-process methods deal with the training data, it is imperative to further link $\mathrm{DE}_{\mathcal{M}_{h}}$ with $\mathrm{DE}_{\mathcal{D}}$. Next, we give the relation between $\mathrm{DE}_{\mathcal{D}_{h}}$ and $\mathrm{DE}_{\mathcal{D}}$ in terms of a bias metric that we refer to as the the \emph{error bias}.

%Next, we give the relation between discrimination in classifier training and discrimination in training data, i.e., $\mathrm{DE}_{\mathcal{D}_{h}}$ and $\mathrm{DE}_{\mathcal{D}}$, in term of the bias of the classifier. The bias metric we derived is what we refer to as the \emph{error bias}.

\begin{definition}[Error Bias]
For any classifier $h$ trained on a training dataset $\mathcal{D}$, the error bias is given by
\begin{equation*}
\varepsilon_{h,\mathcal{D}} = \varepsilon_{1}^{+} - \varepsilon_{2}^{+} - (\varepsilon_{1}^{-} - \varepsilon_{2}^{-}),
\end{equation*}
where $\varepsilon_{1}^{+},\varepsilon_{1}^{-}$ are the percentages of false positives on data with $C=c^{+}$ and $C=c^{-}$ respectively, and $\varepsilon_{2}^{+},\varepsilon_{2}^{-}$ are the percentages false negatives on data with $C=c^{+}$ and $C=c^{-}$ respectively.
\end{definition}

\begin{proposition}\label{thm:dhdf}
For any classifier $h$ that is trained on $\mathcal{D}$, we have
%\begin{equation*}
$\mathrm{DE}_{\mathcal{D}_{h}} - \mathrm{DE}_{\mathcal{D}} = \varepsilon_{h,\mathcal{D}}.$
%\end{equation*}
%Here $\varepsilon_{h,\mathcal{D}} = \varepsilon_{1}^{+} - \varepsilon_{2}^{+} - (\varepsilon_{1}^{-} - \varepsilon_{2}^{-})$, where $\varepsilon_{1}^{+},\varepsilon_{1}^{-}$ are the false positive rates on data with $C=c^{+}$ and $C=c^{-}$ respectively, and $\varepsilon_{2}^{+},\varepsilon_{2}^{-}$ are the false negative rates on data with $C=c^{+}$ and $C=c^{-}$ respectively.
\end{proposition}

\begin{proof}
By definition, $\varepsilon_{1}^{+}$ is given by
\begin{equation*}
\varepsilon_{1}^{+} = \frac{1}{n^{+}}\sum_{\{j:c^{(j)}=c^{+},l^{(j)}=l^{-}\}}\mathbbm{1}_{[h(c^{(j)},\mathbf{z}^{(j)})=l^{+}]},
\end{equation*}
which can be rewritten as
\begin{equation*}
\varepsilon_{1}^{+} = \sum_{\mathbf{z}}\hat{P}(\mathbf{z}|c^{+}) \cdot \mathbbm{1}_{[h(c^{(j)},\mathbf{z}^{(j)})=l^{+}]} \cdot (1-\hat{P}(l^{+}|c^{+},\mathbf{z})).
\end{equation*}
Similarly, $\varepsilon_{2}^{+}$ is given by
\begin{equation*}
\varepsilon_{2}^{+} = \sum_{\mathbf{z}}\hat{P}(\mathbf{z}|c^{+}) \cdot \mathbbm{1}_{[h(c^{(j)},\mathbf{z}^{(j)})=l^{-}]} \cdot \hat{P}(l^{+}|c^{+},\mathbf{z}).
\end{equation*}
Subtracting $\varepsilon_{2}^{+}$ from $\varepsilon_{1}^{+}$, we obtain
\begin{equation*}
\begin{split}
& \varepsilon_{1}^{+} - \varepsilon_{2}^{+} = \\
& \sum_{\mathbf{z}}\hat{P}(\mathbf{z}|c^{+})\Big( \mathbbm{1}_{[h(c^{+},\mathbf{z})=l^{+}]}(1\!-\!\hat{P}(l^{+}|c^{+},\mathbf{z})) \!-\! \mathbbm{1}_{[h(c^{+},\mathbf{z})=l^{-}]}\hat{P}(l^{+}|c^{+},\mathbf{z}) \Big),
\end{split}
\end{equation*}
which is equivalent to 
\begin{equation*}
\varepsilon_{1}^{+} - \varepsilon_{2}^{+} = \sum_{\mathbf{z}}\hat{P}(\mathbf{z}|c^{+}) \cdot \Big( \mathbbm{1}_{[h(c^{+},\mathbf{z})=l^{+}]}-\hat{P}(l^{+}|c^{+},\mathbf{z}) \Big).
\end{equation*}

Similarly for data with $C=c^{-}$, we have
\begin{equation*}
\varepsilon_{1}^{-} - \varepsilon_{2}^{-} = \sum_{\mathbf{z}}\hat{P}(\mathbf{z}|c^{-}) \cdot \Big( \mathbbm{1}_{[h(c^{-},\mathbf{z})=l^{+}]}-\hat{P}(l^{+}|c^{-},\mathbf{z}) \Big).
\end{equation*}

On the other hand, according to Eq. \eqref{eq:plc} and \eqref{eq:plc_} we have
\begin{equation*}
\begin{split}
& \mathrm{DE}_{\mathcal{D}_{h}} \!-\! \mathrm{DE}_{\mathcal{D}} = \sum_{\mathbf{z}} \hat{P}(\mathbf{z}|c^{+}) \Big( \mathbbm{1}_{[h(c^{+},\mathbf{z})=l^{+}]}-\hat{P}(l^{+}|c^{+},\mathbf{z}) \Big) \\
& -\! \sum_{\mathbf{z}} \hat{P}(\mathbf{z}|c^{-}) \Big( \mathbbm{1}_{[h(c^{-},\mathbf{z})=l^{+}]} - \hat{P}(l^{+}|c^{-},\mathbf{z}) \Big) = \varepsilon_{1}^{+} \!-\! \varepsilon_{2}^{+} - (\varepsilon_{1}^{-} \!-\! \varepsilon_{2}^{-}).
\end{split}
\end{equation*}

Letting $\varepsilon_{h,\mathcal{D}} = \varepsilon_{1}^{+} - \varepsilon_{2}^{+} - (\varepsilon_{1}^{-} - \varepsilon_{2}^{-})$ completes the proof.
\end{proof}

Using Proposition \ref{thm:dhdf}, we rewrite Propositions \ref{thm:mhdh} to Theorem \ref{thm:mh} that is easier to interpret and utilize in practice.

%Combining Propositions \ref{thm:mhdh} and \ref{thm:dhdf}, we finally bound $\mathrm{DE}_{\mathcal{M}_{h}}$ in terms of $\mathrm{DE}_{\mathcal{D}}$, $\varepsilon_{h,\mathcal{D}}$, the sample size of $\mathcal{D}$, and the size of hypothesis space $\mathcal{H}$.

%\begin{theorem}\label{thm:mhd}
%For any dataset $\mathcal{D}$ with size of $n$ generated by $\mathcal{M}$, and any classifier $h$ trained on $\mathcal{D}$, the probability that $\mathrm{DE}_{\mathcal{M}_{h}}$ is $t \!+\! \varepsilon_{h,\mathcal{D}}$ far from $\mathrm{DE}_{\mathcal{D}}$ is bounded by
%\begin{equation*}
%P~ \Bigg( |\mathrm{DE}_{\mathcal{M}_{h}} - \mathrm{DE}_{\mathcal{D}}| \leq t \!+\! \varepsilon_{h,\mathcal{D}} \Bigg) \geq 1 - \sqrt{\frac{\pi}{2n}}e^{-\frac{2n^{+}n^{-}}{n}t^{2}}.
%\end{equation*}
%\end{theorem}
%
%Theorem \ref{thm:mhd} can be rewritten as the following corollary.

\begin{theorem}\label{thm:mh}
Given a causal model $\mathcal{M}$, a sample dataset $\mathcal{D}$ and a classifier $h$ trained on $\mathcal{D}$, $\mathrm{DE}_{\mathcal{M}_{h}}$ is bounded by
\begin{equation*}
P~ \Bigg( \left|\mathrm{DE}_{\mathcal{M}_{h}}\right| \leq \left|\mathrm{DE}_{\mathcal{D}}+\varepsilon_{h,\mathcal{D}}\right| + t \Bigg) \geq 1-\delta(t),
\end{equation*}
where the meaning of $\delta(t)$ is same as that in Proposition \ref{thm:mhdh}.
\end{theorem}

%\begin{theorem}\label{thm:mh}
%Given a causal model $\mathcal{M}$, a sample dataset $\mathcal{D}$ and a classifier $h$, for any user-defined parameter $\tau$ ($0 \leq \tau \leq 1$), if $\left|\mathrm{DE}_{\mathcal{D}}+\varepsilon_{h,\mathcal{D}}\right| \leq \tau$, then $\mathrm{DE}_{\mathcal{M}_{h}}$ is bounded by
%\begin{equation*}
%P~ \Bigg( \left|\mathrm{DE}_{\mathcal{M}_{h}}\right| \leq \tau+t \Bigg) \geq 1-\delta,
%\end{equation*}
%where the meaning of $\delta$ is the same as that in Proposition \ref{thm:mhdh}.
%\end{theorem}

%The proof is a straightforward transformation of the combined inequality from Propositions \ref{thm:mhdh} and \ref{thm:dhdf}.

%{\bf Remark.} Proposition \ref{thm:dhdf} derives the classifier bias that is necessary to cause discrimination in prediction. Different from the equal opportunity that requires either the false positive rate or the false negative rate to be equal, the error bias shows that it is the difference between the false positive rate and the false negative rate that really matters. All methods for achieving the equal opportunity can be extended to minimize the error bias. More importantly, this result can provide more room for the trade-off between the fairness/non-discrimination and the accuracy since the error bias is a more general metric than the equal opportunity.

Theorem \ref{thm:mh} gives a criterion of non-discrimination in prediction that incorporates both the discrimination in the training data and the error bias of the classifier, i.e., $\left|\mathrm{DE}_{\mathcal{D}}+\varepsilon_{h,\mathcal{D}}\right|$ being bounded by a threshold $\tau$. It shows that either given a discrimination-free dataset $\mathcal{D}$, i.e., $\left| \mathrm{DE}_{\mathcal{D}} \right|\leq \tau$, or a ``balanced'' classifier, i.e., $\left| \varepsilon_{h,\mathcal{D}} \right|\leq \tau$, we cannot guarantee non-discriminatory prediction. Instead, it requires to ensure that the sum of $\mathrm{DE}_{\mathcal{D}}$ and $\varepsilon_{h,\mathcal{D}}$ is within the threshold. 

\begin{table*}[tbh]\small
\centering
\caption{Measured discrimination before discrimination removal (values larger than threshold are highlighted as bold).}
\label{tab:t1}
\begin{tabular}{|c|c|c|c|c|c|c|}
\hline
\multirow{2}{*}{Size}       & \multirow{2}{*}{$\mathrm{DE}_{\mathcal{M}}$} & \multirow{2}{*}{$\mathrm{DE}_{\mathcal{D}}$} & \multicolumn{2}{c|}{$\mathrm{DE}_{\mathcal{D}_{h}}$} & \multicolumn{2}{c|}{$\mathrm{DE}_{\mathcal{M}_{h}}$} \\ \cline{4-7} 
                            &                                              &                                              & DT                        & SVM                      & DT                        & SVM                      \\ \hline
500                         & \multirow{3}{*}{\textbf{0.130}}                        & $\textbf{0.131} \pm \num{1.6E-3}$                     & $\textbf{0.145} \pm \num{4.1E-3}$  & $\textbf{0.132} \pm \num{8.2E-3}$ & $\textbf{0.139} \pm \num{3.5E-3}$  & $\textbf{0.125} \pm \num{6.8E-3}$ \\ \cline{1-1} \cline{3-7} 
\multicolumn{1}{|c|}{2000}  &                                              & $\textbf{0.131} \pm \num{4.8E-4}$                     & $\textbf{0.129} \pm \num{1.1E-3}$  & $\textbf{0.121} \pm \num{7.4E-3}$ & $\textbf{0.130} \pm \num{9.4E-4}$  & $\textbf{0.120} \pm \num{7.1E-3}$ \\ \cline{1-1} \cline{3-7} 
\multicolumn{1}{|c|}{10000} &                                              & $\textbf{0.129} \pm \num{8.0E-5}$                     & $\textbf{0.138} \pm \num{4.0E-4}$  & $\textbf{0.150} \pm \num{4.3E-3}$ & $\textbf{0.138} \pm \num{3.8E-4}$  & $\textbf{0.150} \pm \num{4.3E-3}$ \\ \hline
\end{tabular}
\end{table*}

\begin{table*}[tbh]\small
\centering
\caption{Measured discrimination after discrimination removal (decision tree as the classifier).}
\label{tab:t2}
\begin{tabular}{|c|c|c|c|c|c|}
\hline
\multirow{2}{*}{Size} & \multicolumn{3}{c|}{Two-phase framework (MSG)}                                                                                        & \multicolumn{2}{c|}{DI}                                               \\ \cline{2-6} 
                      & $\mathrm{DE}_{\mathcal{D}^{*}}$ & $\mathrm{DE}_{\mathcal{M}_{h^{*}}}$ & $\mathrm{DE}_{\mathcal{M}_{h^{*}}}$ (w/o classifier tweaking) & $\mathrm{DE}_{\mathcal{D}^{*}}$ & $\mathrm{DE}_{\mathcal{M}_{h^{*}}}$ \\ \hline
500                   & $\num{0.004} \pm \num{3.7E-6}$ & $\num{0.015} \pm \num{1.0E-3}$     & $\textbf{0.068} \pm \num{4.6E-3}$                                      & $\num{2E-4} \pm \num{1.4E-3}$ & $\textbf{0.092} \pm \num{6.1E-3}$            \\ \hline
2000                  & $\num{0.001} \pm \num{1.7E-7}$ & $\num{0.016} \pm \num{5.3E-4}$     & $\textbf{0.067} \pm \num{4.3E-3}$                                      & $\num{0.001} \pm \num{3.4E-4}$ & $\textbf{0.095} \pm \num{1.6E-3}$            \\ \hline
10000                 & $\num{2E-4} \pm \num{9.7E-9}$ & $\num{0.013} \pm \num{3.3E-4}$     & $\textbf{0.061} \pm \num{3.3E-3}$                                      & $\num{0.001} \pm \num{6.8E-5}$ & $\textbf{0.107} \pm \num{5.4E-4}$            \\ \hline
\end{tabular}
\end{table*}

\section{Remove Discrimination in Prediction}
%This section solves the problem of removing discrimination in prediction: if criterion $\left|\mathrm{DE}_{\mathcal{D}}+\varepsilon_{h,\mathcal{D}}\right| \leq \tau$ is not satisfied for a classifier, how can we meet the criterion through modifying the training data and tweaking the classifier? Denoting by $D^{*}$ a dataset obtained by modifying $\mathcal{D}$, and by $h^{*}$ a new classifier trained on $\mathcal{D}^{*}$, Theorem \ref{thm:mh} %provides the theoretical foundation for building discrimination-free classifiers using pre-process methods, which 
%ensures non-discrimination in prediction as long as $\left|\mathrm{DE}_{\mathcal{D}^{*}} \!+\! \varepsilon_{h^{*},\mathcal{D}^{*}}\right| \leq \tau$ is achieved. Note that when the training data is modified, the error bias of the classifier built on it will also change. Thus, it is difficult to perform the training data modification and the classifier tweaking simultaneous. We propose a framework for modifying the training data and the classifier in two successive phases, as summarized in Algorithm \ref{alg:2pf}.

This section solves the problem of removing discrimination in prediction: if criterion $\left|\mathrm{DE}_{\mathcal{D}}+\varepsilon_{h,\mathcal{D}}\right| \leq \tau$ is not satisfied for a classifier, how can we meet the criterion through modifying the training data and tweaking the classifier? Denote by $D^{*}$ a dataset obtained by modifying $\mathcal{D}$, and by $h^{*}$ a new classifier trained on $\mathcal{D}^{*}$. Note that when the training data is modified, the error bias of the classifier built on it will also change. Thus, it is difficult to perform the training data modification and the classifier tweaking simultaneous. We propose a framework for modifying the training data and the classifier in two successive phases, as summarized in Algorithm \ref{alg:2pf}.

{
\setlength{\interspacetitleruled}{-.4pt}%
\begin{algorithm}
\SetAlgoVlined
\DontPrintSemicolon
%\small
		If $\left| \mathrm{DE}_{\mathcal{D}} \!+\! \varepsilon_{h,\mathcal{D}} \right| \leq \tau$, we are done. Otherwise, modify the labels in the training dataset $\mathcal{D}$ to obtain a modified dataset $\mathcal{D}^{*}$ such that $\left| \mathrm{DE}_{\mathcal{D}^{*}} \right| \leq \tau$. \;
		Train a classifier $h^{*}$ on $\mathcal{D}^{*}$. If $\left| \mathrm{DE}_{\mathcal{D}^{*}} \!+\! \varepsilon_{h^{*},\mathcal{D}^{*}} \right| \leq \tau$, we are done. Otherwise, tweak classifier $h^{*}$ to meet the above requirement.\;
		\caption{Two-phase framework.}
		\label{alg:2pf}	
\end{algorithm}
}

In the first phase, we modify $\mathcal{D}$ to remove the discrimination it contains. The modified dataset $\mathcal{D}^{*}$ can be considered as being generated by a causal model $\mathcal{M}^{*}$ that is different from $\mathcal{M}$ with respect to the modification. Note that if $\left|\mathrm{DE}_{\mathcal{D}^{*}} \!+\! \varepsilon_{h^{*},\mathcal{D}^{*}}\right| \leq \tau$ is achieved, Theorem \ref{thm:mh} ensures the bound of discrimination for $\mathcal{M}^{*}_{h^{*}}$, i.e., the discrimination of $h^{*}$ performed on $\mathcal{M}^{*}$, but not for $\mathcal{M}_{h^{*}}$, i.e., the discrimination of $h^{*}$ performed on the original population. If we only modify the label of $\mathcal{D}$, $\mathcal{M}^{*}$ can be written as
\begin{equation*}
\textrm{Causal Model $\mathcal{M}^{*}$} \quad\quad
\begin{array}{l}
c = f_{C}(pa_{C},u_{C}) \\
z_{i} = f_{i}(pa_{i},u_{i}) \quad i=1,\cdots,m \\
l = f_{L}^{*}(pa_{L}^{*},u_{L}^{*})
\end{array}
\end{equation*}
Then, the causal model of any classifier $h^{*}$ trained on $D^{*}$ and performed on $\mathcal{M}^{*}$ is given by
\begin{equation*}
\textrm{Causal Model $\mathcal{M}^{*}_{h^{*}}$} \quad\quad
\begin{array}{l}
c = f_{C}(pa_{C},u_{C}) \\
z_{i} = f_{i}(pa_{i},u_{i}) \quad i=1,\cdots,m \\
l = h^{*}(c,\mathbf{z})
\end{array}
\end{equation*}
which is equivalent to $\mathcal{M}_{h^{*}}$. Thus, non-discrimination in $\mathcal{M}^{*}_{h^{*}}$ also means non-discrimination in $\mathcal{M}_{h^{*}}$. On the other hand,  if we modify attributes other than $L$, since the new unlabeled data is drawn from the original population, $\mathcal{M}_{h^{*}}$ is inconsistent with $\mathcal{M}^{*}_{h^{*}}$. As a result, the non-discrimination result of the training data cannot be applied to the prediction of the new data. Therefore, we have the following corollary derived from Theorem \ref{thm:mh}.

%Although a great number of pre-process methods have been proposed for removing discrimination from the training data, the following theorem shows that, only the methods that only modify the labels of the data can ensure non-discrimination in the prediction by using our framework. %The proof adopts the assumption that $L$ has no child.

\begin{corollary}\label{thm:mh*}
Let $\mathcal{D}^{*}$ be a modified dataset from $\mathcal{D}$, and $h^{*}$ be a new classifier trained on $\mathcal{D}^{*}$. If $\mathcal{D}^{*}$ only modifies the labels, then $\left|\mathrm{DE}_{\mathcal{D}^{*}} \!+\! \varepsilon_{h^{*},\mathcal{D}^{*}}\right| \leq \tau$ is a sufficient condition to achieve
\begin{equation*}
P~ \bigg( \left|\mathrm{DE}_{\mathcal{M}_{h^{*}}}\right| \leq \tau+t \bigg) \geq 1-\delta(t),
\end{equation*}
where the meaning of $\delta(t)$ is same as that in Proposition \ref{thm:mhdh}.
\end{corollary}

In the second phase, we make modifications to $h^{*}$ to reduce the error bias. Although dealing with a different fairness criterion, existing methods for balancing the misclassification rates (e.g., \cite{hardt2016equality}) can be easily extended for solving this problem. For the purpose of evaluating the correctness of our theoretical results, here we use a simple algorithm \emph{RandomFlip} for reducing the error bias that can be applied to any classifier. 
After the classifier makes a prediction, \emph{RandomFlip} randomly flips the predicted label with certain probability $p^{+}$ (resp. $p^{-}$) if the individual has $C=c^{+}$ (resp. $C=c^{-}$) to achieve $\left|\mathrm{DE}_{\mathcal{D}^{*}} \!+\! \varepsilon_{h^{*},\mathcal{D}^{*}}\right| \leq \tau$, where $p^{+}$ can be computed according to the prediction of $h^{*}$ over $\mathcal{D}^{*}$. Denoting $\sigma = \tau - \left| \mathrm{DE}_{\mathcal{D}^{*}} \right|$, it suffices to make $|\varepsilon_{1}^{+}-\varepsilon_{2}^{+}|\leq \sigma/2$ and $|\varepsilon_{1}^{-}-\varepsilon_{2}^{-})|\leq \sigma/2$. Assume that $\varepsilon_{1}^{+}-\varepsilon_{2}^{+} > \sigma/2$, then it can be easily shown that $p^{+}$ should satisfy $( \varepsilon_{1}^{+} - \varepsilon_{2}^{+} - \sigma/2 ) \left( \frac{n^{+}}{\mathit{fp}+\mathit{tp}} \right) \leq p^{+} \leq (\varepsilon_{1}^{+} - \varepsilon_{2}^{+}) \left( \frac{n^{+}}{\mathit{fp}+\mathit{tp}} \right)$. Similar result can be obtained if $\varepsilon_{1}^{+}-\varepsilon_{2}^{+} < -\sigma/2$.

%Although the two-phase framework is proposed for pre-process methods, Theorem \ref{thm:mh} also applies to in-process and post-process methods, which are equivalent to only dealing with $\varepsilon_{h,\mathcal{D}}$, either through adding fair constraints into the learning process or directly modifying the predicted labels produced by the classifier. However, as discussed in the related work, both in-process and post-process methods have their respective limitations. For in-process methods, since $\varepsilon_{h,\mathcal{D}}$ is a non-convex function, the learning algorithm needs to use a convex surrogate function for the minimization (e.g., in \cite{}). Thus, additional unbounded bias might be further introduced due to the approximation errors associated with the surrogate function. For post-process methods, they are restricted to those who can modify the predicted label of each individual independently. Thus, the methods that map the whole dataset to a new non-discriminatory dataset cannot be applied. This means that a number of causal-based method (e.g., \cite{}) cannot be adopted for post-process.%To the best of our knowledge, so far no causal-based method can deal with each individual independently(?).

\section{Empirical Evaluation}

\subsection{Experimental Setup}
In this section, we conduct experiments to evaluate our theoretical results. For simulating a population, we adopt a semi-synthetic data generation method. We first learn a causal model $\mathcal{M}$ for a real dataset, the Adult dataset \cite{adultdataset}, and treat it as the ground-truth. We then generate the training data $\mathcal{D}$ based on $\mathcal{M}$. The causal model is built using the open-source software Tetrad \cite{tetrad}. %All data and source codes are available at \url{https://}.

The Adult dataset consists of 11 attributes including \texttt{age}, \texttt{education}, \texttt{sex}, \texttt{occupation}, \texttt{income}, \texttt{marital\_status} etc. Due to the sparse data issue, we binarize each attribute's domain values into two classes to reduce the domain sizes. 
%We use three tiers in the partial order for temporal priority: \texttt{sex}, \texttt{age}, \texttt{native\_country}, \texttt{race} are defined in the first tier, \texttt{edu\_level} and \texttt{marital\_status} are defined in the second tier, and all other attributes are defined in the third tier. 
%The constructed causal network is not presented in the paper due to space limitation. 
We treat \texttt{sex} as $C$ and \texttt{income} as $L$. 
%Since original Adult dataset has nearly zero discrimination (measured as 0.018), before learning the causal model we manually introduce biases into the data by randomly changing the labels of some females from positive to negative and changing the labels of some males from negative to positive. 
The discrimination is measured as $0.13$ in $\mathcal{M}$, i.e., $\mathrm{DE}_{\mathcal{M}} = 0.13$.
Based on the underlying distribution of $\mathcal{M}$, we generate a number of training data sets with different sample sizes. 

When constructing discrimination-free classifiers using the two-phase framework, we select one representative data modifying algorithm that only modifies $L$, the \emph{Massaging} (MSG) algorithm \cite{kamiran2009classifying}. For other algorithms, we will evaluate their performance in preserving data utility in the future work.
For comparison, we also include an algorithm that modifies $\mathbf{Z}$, the \emph{Disparate Impact Removal} (DI) algorithm \cite{adler2016auditing}.
%two data modifying algorithms are used for removing discrimination from the training data: the \emph{Massaging} (MSG) algorithm \cite{kamiran2009classifying} that modifies $L$ and the \emph{Disparate Impact Removal} (DI) algorithm \cite{adler2016auditing} that modifies $\mathbf{R}$. 
The proposed \emph{RandomFlip} algorithm is used for tweaking the classifier. We assume a discrimination threshold $\tau = 0.05$, i.e., we want to ensure that the discrimination in prediction is not larger than 0.05.

\subsection{Experiment Results}
We first measure the discrimination in each training data set, i.e., $\mathrm{DE}_{\mathcal{D}}$. Then, we learn the classifier $h$ from the training data where two classifiers, decision tree (DT) and support vector machine (SVM) are used. By replacing the labels in the training data with the labels predicted by the classifier, we obtain $\mathcal{D}_{h}$ whose discrimination is measured as $\mathrm{DE}_{\mathcal{D}_{h}}$. Finally, we measure the discrimination in prediction, i.e., $\mathrm{DE}_{\mathcal{M}_{h}}$ according to Proposition \ref{thm:compute_mh}. For each sample size, we repeat the experiments 100 times by randomly generating 100 different sets of training data. We report the average and variance of each measured quantity of discrimination.

The results are shown in Table \ref{tab:t1}. As expected, with the increase of the sample size, the difference between $\mathrm{DE}_{\mathcal{M}}$ and $\mathrm{DE}_{\mathcal{D}}$ decreases as shown by the variance. Since $\mathrm{DE}_{\mathcal{D}}>0.05$, the training data contains discrimination. As a result, both the training data with predicted labels, i.e., $\mathcal{D}_{h}$, and the prediction, i.e., $\mathcal{M}_{h}$, also contain discrimination.

To show the effectiveness of the two-phase framework, we first apply MSG to completely remove the discrimination in the above training data, obtaining the modified training data $\mathcal{D}^{*}$. Then, a decision tree $h^{*}$ is built on $\mathcal{D}^{*}$, and the \emph{RandomFlip} algorithm is executed to tweak the classifier so that the error bias is less than 0.05, i.e., $\left| \varepsilon_{h^{*},\mathcal{D}^{*}} \right| \leq 0.05$. Finally, we measure the discrimination in $\mathcal{M}_{h^{*}}$. For comparison, the same process is also performed for DI.

%we first apply the data modifying algorithms MSG and DI to completely remove the discrimination in the above training data, obtaining two modified training data $\mathcal{D}^{*}_{1}$ and $\mathcal{D}^{*}_{2}$. Two decision trees $h^{*}_{1}$ and $h^{*}_{2}$ are built on $\mathcal{D}^{*}_{1}$ and $\mathcal{D}^{*}_{2}$ respectively. We then apply the \emph{RandomFlip} algorithm to tweaking both classifiers so that the error bias of each classifier is less than 0.05, i.e., $\varepsilon_{h^{*}_{1},\mathcal{D}^{*}_{1}}\leq 0.05$ and $\varepsilon_{h^{*}_{2},\mathcal{D}^{*}_{2}}\leq 0.05$. Finally, we measure the discrimination in $\mathcal{M}_{h^{*}_{1}}$ and $\mathcal{M}_{h^{*}_{2}}$. 

The results are shown in Table \ref{tab:t2}. By using the two-phase framework, discrimination is removed from the training data as shown by $\mathrm{DE}_{\mathcal{D}^{*}}$, and more importantly, removed from the prediction as shown by $\mathrm{DE}_{\mathcal{M}_{h^{*}}}$. We also see that, if the classifier tweaking is not performed, the prediction still contains discrimination. However, for DI, even when the discrimination is removed from the training data, and the error bias in the classifier is also removed, there still exists discrimination in prediction. 
These results are consistent with our theoretical conclusions.

\section{Conclusions}
In this paper, we addressed the limitation of the pre-process methods that there is no guarantee about the discrimination in prediction. Our theoretical results show that: (1) only removing discrimination from the training data cannot ensure non-discrimination in prediction for any classifier; and (2) when removing discrimination from the training data, one should only modify the labels in order to obtain a non-discrimination guarantee. Based on the results, we developed a two-phase framework for constructing a discrimination-free classifier with a theoretical guarantee. The experiments adopting a semi-synthetic data generation method demonstrate the theoretical results and show the effectiveness of our two-phase framework.

\bibliographystyle{named}
\bibliography{ijcai18}

\end{document}